\documentclass[journal,onecolumn]{IEEEtran}

\usepackage[utf8]{inputenc} 
\usepackage[T1]{fontenc}    
\usepackage{hyperref}       
\usepackage{url}            
\usepackage{booktabs}       
\usepackage{amsfonts}       
\usepackage{nicefrac}       
\usepackage{microtype}      
\usepackage{amsmath,amssymb,amsthm,cite,graphicx}
\usepackage{subcaption}
\usepackage[ruled,linesnumbered]{algorithm2e}
\usepackage{booktabs}

\newtheorem{theorem}{Theorem}
\newtheorem{lemma}{Lemma}
\newtheorem{corollary}{Corollary}
\newtheorem{assumption}{Assumption}
\theoremstyle{definition}
\newtheorem{definition}{Definition}
\theoremstyle{remark}
\newtheorem*{remark}{Remark}

\title{Structure Learning from Time Series\\ with False Discovery Control}

\author{
     Bernat Guillen Pegueroles$^{1,2}$\thanks{This work was completed under the auspices of the IBM Science for Social Good initiative.}, Bhanukiran Vinzamuri$^2$, Karthikeyan Shanmugam$^2$, \\Steve Hedden$^3$, Jonathan D.\ Moyer$^3$, and Kush R.\ Varshney$^2$ \vspace{4pt} \\
  $^1$Program in Applied and Computational Mathematics, Princeton University,  Princeton, NJ 08544 \\
  \texttt{bernatp@princeton.edu} \vspace{4pt}\\
   $^2$IBM Research, Yorktown Heights, NY 10598 \\
  \texttt{\{bhanukiran.vinzamuri@,karthikeyan.shanmugam2@,krvarshn@us.\}ibm.com}  \vspace{4pt}\\
  $^3$Frederick S. Pardee Center for International Futures, Josef Korbel School of International Studies,\\ University of Denver, Denver, CO 80210 \\
 \texttt{\{steve.hedden,jonathan.moyer\}@du.edu} \\
}

\begin{document}

\maketitle

\begin{abstract}
We consider the Granger causal structure learning problem from time series data. Granger causal algorithms predict a `Granger causal effect' between two variables by testing if prediction error of one decreases significantly in the absence of the other variable among the predictor covariates. Almost all existing Granger causal algorithms condition on a large number of variables (all but two variables) to test for effects between a pair of variables. We propose a new structure learning algorithm called MMPC-p inspired by the well known MMHC algorithm for non-time series data. We show that under some assumptions, the algorithm provides false discovery rate control. The algorithm is sound and complete when given access to perfect directed information testing oracles. We also outline a novel tester for the linear Gaussian case. We show through our extensive experiments that the MMPC-p algorithm scales to larger problems and has improved statistical power compared to existing state of the art for large sparse graphs. We also apply our algorithm on a global development dataset and validate our findings with subject matter experts.
\end{abstract}

\section{Introduction}
\label{sec:intro}

Decisions to set policies and conduct interventions ought to be supported by evidence and scenario analysis.  Supporting decision making in this way, however, is difficult since it is usually not possible to conduct randomized controlled trials over policy decisions such as a country's investment in sanitation or primary education and see the effect on indicators such as number of patents produced or tuberculosis mortality rate.  Understanding the causal structure of the many different variables, indicators, and possible interventions relevant to such decisions is challenging because of their intricate interdependencies and the small cardinality of noisy samples coupled with a large number of variables.  One approach for developing such understanding is through painstaking theorization and validation of small sets of variables over many years \cite{HughesH2006}. An alternative that we focus on in this paper is to estimate the structure of a Bayesian network from an observed time series.

Existing approaches for the problem of structure learning from time series observations include the Granger Causality (GC) algorithm \cite{barnett2014mvgc}. This has been recently formalized in terms of directed information graphs \cite{quinn2015directed,eichler2012graphical} as a Bayesian network structure recovery problem on time series. The GC approach has good statistical properties because it conditions on \textit{all} other variables to isolate the pair in question \cite{cooklearning}. However, with finite samples and very large number of variables, the statistical power of the algorithm significantly reduces due to large conditioning sets. Inspired by the Peter and Clark (PC) algorithm for causal discovery for non-time series data, recently proposed variants perform tests that condition on a reduced subset, beginning with a complete graph and pruning it with pairwise tests  \cite{moneta2011causal}; this approach yields many false positives while also having scaling issues.

In this paper, we propose a new algorithm, MMPC-p, that is scalable and has provably strong p-value control to prevent false discoveries using techniques from \cite{strobl2016estimating,armen2011unified}.  This proposed algorithm is inspired by the MMHC algorithm for causal structure learning in non-time series observational settings \cite{tsamardinos2006max}.  It begins with an empty graph, adds edges to form candidate parent sets, and subsequently prunes them in a two-phase approach.  We show it to be sound, complete and equip it with false discovery rate (FDR) control under assumptions we describe in the sequel.  

The proposed MMPC-p algorithm relies on some form of independency testing on pairs of random processes.  Due to autocorrelations across time, we cannot use conditional independence tests directly.  We consider two testers based on directed information, an information-theoretic measure of predictive information flow between processes for linear Gaussian models. The first is a na\"{\i}ve directed information test that ignores correlations across time but requires fewer samples.  The second computes directed information in a more principled way but requires a greater number of samples.

We conduct a detailed comparison study of GC, PC, and MMPC-p on generated data and find MMPC-p to have better scaling and error control empirically. From our synthetic experiments, we find that MMPC-p has higher statistical power for sparse large graphs than the alternatives: GC and PC. We also apply MMPC-p to a real-world data set of global development indicators from 186 countries over more than fifty years and compare the learned causal relationships to the validated relationships in the International Futures (IFs) program \cite{HughesH2006}. There are systematic differences in the two sets of relationships that we detail later, but the ones found by the proposed algorithm have some validity from the policy expert perspective. 

The main contributions of this work are: (1) an MMPC-p algorithm for time series data inspired by the MMHC algorithm, (2) a method to control false discoveries with our approach under weak assumptions on Type II error, (3) exhaustive experiments comparing the performance of MMPC-p with the modified PC and modified GC algorithms \cite{cooklearning} (we show that MMPC-p performs well for large sparse graphs in terms of both omission and commission errors), and (4) a case study on a global development dataset with input from subject matter experts.  

\section{Related Work}
Most recent work in causal structure learning has focused on issues related to undersampling. In \cite{hyttinen2017constraint,plis2015rate}, causal time scale structures are learned from subsampled measurement time scale graphs and data. A variant of this was studied in \cite{gong2017causal}, where the authors address the issue of causal structure learning from temporally aggregated time series. We consider the learning problem at measurement time scales only. Recent work \cite{cooklearning} has drawn attention to the need for evaluating algorithms on the measurement time scale problem, which is used by some of the algorithms that deal with under sampling. Regression-based methods have also been used for estimating the causal impact \cite{brodersen2015inferring}, which is quantified as the counterfactual response in a synthetic control setting. In contrast, here we focus on the structure learning aspect of the problem. In \cite{moneta2011causal}, the PC algorithm is extended for time series under the assumption that the measurement and time scales were approximately equal. Another variant called modified PC has been presented in \cite{cooklearning}. We actually compare our results to this variant in the empirical section. In \cite{barnett2014mvgc}, the authors present techniques for estimating multivariate Granger causality from time series data, both unconditional and conditional, in the time and frequency domains.  In contrast, \cite{lozano2009spatial} explores combining graphical modeling with Granger causality to address climate change problems. These papers use the Granger causal algorithm that conditions on all the variables but the pair of variables in questions. We actually compare to a variant of these approaches as described in \cite{cooklearning}. More recently, the importance of FDR control is being emphasized in causal structure learning problems. An approach for Bayesian network structure learning with FDR control was presented in \cite{armen2011unified}. In \cite{chaudhry2017uncertainty}, the authors present p-value estimates for high-dimensional Granger causal inference under assumptions of sub-Gaussianity on the coefficients. PC-p \cite{strobl2016estimating} is an extension of PC which computes edge-specific p-values and controls the FDR across edges. In contrast to these algorithms, our MMPC-p is for time series, works under general assumptions and is inspired by the MMHC algorithm. We also formally prove FDR control guarantees and back it up with results in our empirical section.

\section{Formal Problem Definition and Preliminaries}\label{formal}

\textbf{Notation.}
Consider $m$ random processes over time slots $0 \ldots T$ such that the $i$th random process sequence is denoted by:
\[\mathbf{X}_i = (X_{i,0} \ldots X_{i,T}), ~i \in [1:m].\]

Further, let $\mathbf{X}_i^{(t)}=(X_{i,0} \ldots X_{i,t})$ denote the sequence up to time $t$.
Let $\mathbf{X}_i^{(t_0,t_1)}$ denote the $i$th random process from time $t_0$ to $t_1$. Consider a subset of random processes $A \subset [1:m]$. Then, the random variables of all the processes in $A$ from time $t_0$ to $t_1$ are denoted $\mathbf{X}_A^{(t_0,t_1)}$. $\mathbf{X}_{A,t}$ denotes the random variables belonging to the set of processes $A$ at time $t$. Let $\mathbf{X}_A$ and $\mathbf{X}_A^{(t)}$ denote the quantities analogous to the single random process case as described above.

We will primarily consider random processes that take values in a finite alphabet. This is to simplify the presentation, avoiding all measure-theoretic issues. The experiments, however, are performed with respect to real-valued random processes. We make the following assumptions.
\begin{assumption}[Strict Causality]\label{assum1}
The $m$ random processes follow the dynamics:
 \begin{align}\label{fac}
   P_{\mathbf{X}_{}[1:m]}\left(\mathbf{x}_{[1:m]} \right)= \prod \limits_{t=1}^T \prod \limits_{i \in [1:m]} P_{X_{i,t}\lvert \mathbf{X}_{[1:m],t-1} }(x_{i,t}\lvert \mathbf{x}_{[1:m],t-1}) .
 \end{align}
\end{assumption}

 The dynamics are order-1 Markov. This supposes that there are no instantaneous interactions in the system conditioned on the past, i.e.\ the system is strictly causal.
 
Following \cite{quinn2015directed}, let us denote the above causal conditioning over time in \eqref{fac} using:
\begin{align}
  P_{\mathbf{X}_{[1:m]}}\left(\mathbf{x}_{[1:m]} \right) = \prod \limits_{i \in [1:m]} 
   P_{X_i\lVert X_{[1:m]}} (\mathbf{x}_i \lVert \mathbf{x}_{[1:m]}).
\end{align}
Here, the notation $P(\cdot\lVert\cdot)$ subsumes the recursive causal conditioning over time in (\ref{fac}).
 
\begin{assumption}[Causal Sufficiency]\label{assum2}
There are no hidden confounders and all variables are measured.
\end{assumption}
\begin{assumption}[Positivity]\label{assum3}
The joint distribution satisfies: $P_{\mathbf{X}_{[1:m]}}(\mathbf{x}_{[1:m]}) >0$ for all possible realizations $\mathbf{x}_{[1:m]}$ in the domain.
\end{assumption}
This is a sort of `faithfulness assumption': the data does not exhibit any near-deterministic relationships. We review relevant results from \cite{quinn2015directed} and \cite{eichler2012graphical} under the above assumptions.

\begin{definition}
   \textit{Causally conditioned directed information} from random process $\mathbf{X}_i$ to $\mathbf{X}_j$ conditioned on the random processes in the set $A$ is given by:
   \begin{equation}
     I(\mathbf{X}_j \rightarrow \mathbf{X}_i \lVert \mathbf{X}_A)
     = \frac{1}{T} \left(\sum \limits_{t=1}^T I \left(X_{j}^{(t-1)}; X_{i,t} \lvert X_{i}^{(t-1)},\mathbf{X}_A^{(t-1)}\right) \right).
   \end{equation}
 Here, $I(\cdot ; \cdot \lvert \cdot)$  represents the standard conditional mutual information measure in information theory.
\end{definition}
In other words, it is the time average of the mutual information between process $j$ until time $t-1$ and process $i$ at time $t$ given the past of processes in $i \cup A$ until time $t-1$. It is related to Granger causality, signifying the reduction in prediction loss that process $j$ until $t-1$ gives over and above the processes in $i \cup A$ until time $t-1$. The notion is exact for prediction under log loss. However, \cite{jiao2015justification} presents arguments as to why the log loss is the correct metric for measuring value of extra side information in prediction as only this measure satisfies a data-processing axiom.

\begin{definition}\label{DGraph}
 Directed Information (DI) graph $G$ is a DAG $G=(V=\{\mathbf{X}_1, \ldots\mathbf{X}_m\},E)$ associated with the $m$ random processes is defined as follows: a directed edge $(\mathbf{X}_i,\mathbf{X}_j) \in E$ iff $I(\mathbf{X}_i \rightarrow \mathbf{X}_j \lVert \mathbf{X}_{[1:m]-\{i,j\}})>0$.
\end{definition}
This is a graph where every node is a random process. We interchangeably use $i$ and $\mathbf{X}_i$ when talking about nodes in the graph $G$. Let $\mathrm{Pa}(i) = \{j: (j,i) \in E \}$ be the set of directed parents of node $i$ in the DI graph $G$. Let $\mathrm{Ch}(i)=\{j: (i,j) \in E\}$ be the set of children of $i$.

\begin{theorem}[\cite{quinn2015directed,eichler2012graphical}]
Let $\mathrm{Pa}(i)$ be the set of directed parents according to the DI graph. Then, if the positivity condition holds for all $m$ random processes over time and if the system is strictly causal, then almost surely:
 \[P_{\mathbf{X}_{[1:m]}}(\mathbf{x}_{[1:m]}) = \prod \limits_{i \in [m]} P_{\mathbf{X}_i \lVert \mathbf{X}_{\mathrm{Pa}(i)}} \left ( \mathbf{x}_i \lVert \mathbf{x}_{\mathrm{Pa}(i)} \right). \]
\end{theorem}

\begin{corollary}[Local Causal Markov Property \cite{quinn2015directed,eichler2012graphical}] When the system of $m$ random processes satisfies the positivity constraint and satisfies strict causality:
  \[ I \left(\mathbf{X}_{A} \rightarrow \mathbf{X}_{i} \lVert   \mathbf{X}_{\textbf{Pa}(i)}\right) =0 , ~\forall A: A \subseteq [1:m]-\{i \cup \mathrm{Pa}(i)\}.\]
\end{corollary}

\begin{assumption}\label{assum4}
  If $I(\mathbf{X}_i \rightarrow \mathbf{X}_j \lVert \mathbf{X}_{[m]-\{i,j\}})>0$, then $I(\mathbf{X}_i \rightarrow \mathbf{X}_j \lVert \mathbf{X}_{A}>0)$, for all sets $A \subset [1:m]-\{i,j\}$.
\end{assumption}

\section{Algorithm: MMPC-p}

Inspired by the MMHC algorithm for observational causal discovery \cite{tsamardinos2006max} with i.i.d.\ data, we introduce an adaptation called the MMPC-p algorithm (max-min parents) for Granger causality. The MMPC-p algorithm uses a DI Tester as an oracle instead of a Conditional Independence (CI) Tester. We will prove an upper bound on the p-values of the edges obtained and show that p-value control is possible in this case under some weak assumptions. 

\textbf{DI Testing Oracle }$\mathrm{DI}(i,j,A)$: This DI testing function outputs the probability (or p-value) of the event $I(\mathbf{X}_i \rightarrow \mathbf{X}_j \lVert \mathbf{X}_A)=0$ for any $A \subset [1:m]$ given the dataset. We will first assume this oracle that outputs p-value to specify our MMPC-p algorithm.

Let us assume we have a measure of association 
\begin{align}\label{assoc}
\mathrm{Assoc}_{\alpha}(i\to j; A) = \alpha- \min(\alpha,DI(i,j,A)).
\end{align}

Define the functions max-min association and argmax-min association as follows:
\begin{align*}
&\mathrm{mma}_{\alpha}(j; A) = \max_{i \neq j}\min_{F \subset A} \mathrm{Assoc}_{\alpha}(i\to j; F)\\
&\mathrm{amma}_{\alpha}(j;A) = \arg\max_{i \neq j} \min_{F \subset A} \mathrm{Assoc}_{\alpha}(i\to j; F).
\end{align*}

We now describe the MMPC-p algorithm presented in Algorithm \ref{algo:mmpcp}. It consists of two phases: the first phase picks candidate parents while the second prunes the list of candidate parents picked in the first phase.

\begin{algorithm}
 \KwData{$j, V, \mathrm{Data},\alpha$}
 \KwResult{$\mathrm{Pa}(j)$, the parents of $X_j$}
 $CP(j) = \emptyset$, $~{\cal P}(i \to j)=\emptyset, ~\forall i \neq j$ \;
 \tcc{\textbf{Phase I}}
 \Repeat{assocP = 0}{
  $P=\mathrm{amma}_{\alpha}(X_j; CP(j))$\;
  $\mathrm{assocP}=\text{mma}_{\alpha}(j; CP(j)$\;
  \uIf{$\text{assocP}>0$}{
   $CP(j) = CP(j) \cup P$\;
  }  
 }
 \tcc{\textbf{Phase II}}
 \For{$Y\in CP(j)$}{
 \label{assocY} $\mathrm{assocY} = \min \limits_{F \subset CP(j) \setminus Y} \mathrm{Assoc}_{\alpha}(Y\to j; F)$\;
  
  $F_{\mathrm{min}} = \arg \min \limits_{F \subset CP(j) \setminus Y} \mathrm{Assoc}_{\alpha}(Y\to j; F)$ \;
  
  \uIf{$assocY \neq 0$}{
  \For{$F \subset CP(j) \setminus Y$}{
  $p$ $\leftarrow$ p-value from $DI(Y \to j; F)$\;
  \uIf{$p \leq \alpha$} {
  Insert $p$ into  $\mathcal{P}(Y \to j)$ \;
  }
  }
  }
  \uElse{}{
  $CP(j)= CP(j)\setminus Y$\;
  Empty $\mathcal{P}(Y \to j)$ \;
  }
 }
 $\mathcal{P}(Y \to j)=\mathrm{max} \lbrace \mathcal{P}(Y \to j) \rbrace$\;
 \KwRet{$CP(j),{\cal P}$}
\caption{MMPC-p}
\label{algo:mmpcp}
\end{algorithm}

\begin{assumption}\label{assum5}
If $I(\mathbf{X}_i \rightarrow \mathbf{X}_j \lVert \mathbf{X}_A) >0$, $DI(i,j;A) < \alpha$ for $\alpha$ used in Algorithm \ref{algo:mmpcp}. 
\end{assumption}

The above says that Type II errors are small. Related to the faithfulness assumption, it means there are no very weak dependencies in the system. Similar assumptions have been made for p-value control for causal inference with i.i.d.\ data \cite{strobl2016estimating}.

\begin{lemma}\label{lem:phase1}
Type II error less than $\alpha$ (Assumption \ref{assum5}) implies that $\mathrm{Pa}(j) \subseteq CP(j)$ after Phase I of MMPC-p.   
\end{lemma}
\begin{proof}
We present a proof by contradiction. Suppose $v \in \mathrm{Pa}(j)$, then $I(v \to j| S) > 0$ $\forall S$. This implies that $DI(v,j;S) < \alpha,~\forall S$ by Assumption \ref{assum5}. Suppose $v$ is not included in $CP(j)$ and Phase I of MMPC-p completes. This implies that when looking at $v$, there exists a subset $S$ such that $\mathrm{Assoc}_{\alpha}(v \to T|S)=0$. This implies that $DI(v,j;S) \geq \alpha$ for that subset $S$ yielding a contradiction. Therefore, node $v$ will be included in $CP(j)$ at the end of Phase I of MMPC-p.
\end{proof}

\begin{lemma}[\cite{strobl2016estimating}]
Consider $m$ CI testers and the following null and alternative hypothesis
\begin{equation}
\label{eqn:mmhchypothesis}
H_0: \text{At least one CI oracle outputs independent}; \quad H_1: \text{All CI oracles output dependent.}
\end{equation}

Assuming that the $i$th CI oracle outputs independent of all other oracles, we can bound the p-value \eqref{eqn:mmhchypothesis} as $p \le \max_{j=1,\ldots,m} p_j$.

\end{lemma}
\begin{theorem}
For all the edges $A \rightarrow T$ that finally remain after Phase II of MMPC-p the $\max(\mathrm{pvalue},\alpha) \le \mathcal{P}(A \rightarrow T).$
\end{theorem}
\begin{proof}
After completion of Phase I, we wish to test whether the edge is present by conducting independence tests. We construct a hypothesis test with the following null and alternative:
\begin{equation}
\label{eqn:mmhchypothesis2}
H_0: A \rightarrow T \quad \text{is absent}; \quad H_1: A \rightarrow T \quad \text{is present.}
\end{equation}

where $T$ represents the target node. According to Lemma~\ref{lem:phase1} and referring to lines 9 to 14 of Algorithm~\ref{algo:mmpcp}, the p-value for parents for a given target $T$ will always be less than $\alpha$, and would never be dropped.

Since all parents are in $CP$, testing for $H_1$ in \eqref{eqn:mmhchypothesis2} is equivalent to testing $H_1$ in \eqref{eqn:mmhchypothesis}. Similarly, testing for $H_0$ in \eqref{eqn:mmhchypothesis2} is equivalent to testing $H_0$ in \eqref{eqn:mmhchypothesis}. Hence, the hypothesis test defined in \eqref{eqn:mmhchypothesis2} is equivalent to the hypothesis test defined in \eqref{eqn:mmhchypothesis}. Hence, our Algorithm~\ref{algo:mmpcp}, keeps track of the p-value by bounding $\max(\mathrm{pvalue},\alpha) \le \max\left[\max \limits_{\forall F \mathrm{in~ line~ 12~for} A} DI(A \rightarrow T;F),\alpha \right]=\mathcal{P}(A \rightarrow T)$. 
\end{proof}

\textbf{FDR Control:}We define $FDR_{BY}(\beta)$ given by:
\begin{align}\label{eq:control}
FDR_{BY}(\beta)&\triangleq \frac{m \beta \Sigma_{i=1}^{m}\frac{1}{i}}{\max\{R,1\}}
\end{align}
where $R$ is the number of edges retained at the end of Phase II of MMPC-p. The value $\beta^{*}$ satisfies:
\begin{align}\label{eq:controlopt}
\beta^{*} \triangleq \underset{\beta}{\mathrm{arg max}}\{FDR(\beta) \le q\}.
\end{align}
Given a target false positive rate $q$, deleting directed edges $A \to T$ whose $\mathcal{P}(A \to T) > \beta^{*}$ ensures consistent FDR control provided $\beta^{*} \leq \beta$ in Algorithm~\ref{algo:mmpcp}. 

\begin{theorem}
The MMPC-p algorithm with a perfect DI oracle is sound and complete.
\end{theorem}
\begin{proof}
A perfect DI oracle means that $DI(i,j;F)=1$ if $I(\mathbf{X}_i \to \mathbf{X}_j \lVert F) =0$ and $DI(i,j;F)=0$ if $I(\mathbf{X}_i \to \mathbf{X}_j \lVert F) > 0$. With this strong assumption and Lemma \ref{lem:phase1}, for any node $j$, $\mathrm{Pa}(j) \subseteq CP(j)$ after the first phase of MMPC-p.

Next, we show that if $i$ is not in $\mathrm{Pa}(j)$ then $i \notin CP(j)$ after the second phase. The reason is that one of the subsets of $CP(j)$ after the first phase has to equal $\mathrm{Pa}(j)$ (as shown in the previous paragraph). Suppose, $i \notin \mathrm{Pa}(j)$, then we know that $I(\mathbf{X}_i \to \mathbf{X}_j \lVert \mathrm{Pa}(j)) =0$. Therefore, $DI(i,j;\mathrm{Pa}(j))=1$. This means that for any $\alpha>0$, $assocY=0$ at Line \ref{assocY} when $Y=i$ if $i \in CP(j)$ after Phase I. This would cause $i$ to be discarded in the second phase.
\end{proof}

\begin{remark}
This algorithm is much simpler than the one it is inspired by: MMHC. The definition of DI and the role of time in its computation simplify the algorithm and its proof. Furthermore, we do not have problems of ``descendants'' staying after the two phases of the algorithm (there is a second part of the MMHC algorithm in the original paper where pairs were only considered if $i \in CP(j)$ and $j \in CP(i)$, that is not required here). However, the algorithm still retains the robustness of MMHC.
\end{remark}

\section{DI Testers for Linear Models and Gaussian Processes}

Let the scalar variable $X$ follow a memory-$1$ autoregressive linear model with i.i.d.\ Gaussian noise given by $X(t+1) = \Phi(t)X(t) + \xi(t)$ where $\xi(t) \sim {\cal N}(0,\sigma^2)$ and  $\xi(t)$ is independent across time. Generalizing to a set of random variables with an underlying Granger causal graph (the DI graph) in the sense of Section \ref{formal}. Now given the DI graph i.e., for variable $i$ there is a set $Pa(i)$ (that does not depend on $t$) such that $X_{i, t+1} = \sum_{j \in Pa(i)}\Phi_{ij}X_{j, t} + \Phi_{ii} X_{i,t}+ \xi_i(t)$ where $\phi_{ij}$'s are the coupling coefficients. Let $T$ denote the number of time points sampled for every variable $i$. Let the number of i.i.d.\ copies of these time series is $N$. Every variables essentially is observed $NT$ times, $T$ across time for each i.i.d.\ sample. For jointly Gaussian autocorrelated time series processes, DI can be computed by \cite{amblard2009relating}. 
$I(\mathbf{X}_i \rightarrow \mathbf{X}_j \lVert \mathbf{X}_A) = \frac{1}{2}\log \frac{\epsilon_\infty^2(j,A)}{\epsilon_\infty^2(j,A \cup i)}$. Here, $\epsilon_{\infty}^2(i,A)$ is the asymptotic prediction error of $X_{i,t}$ given the past of the process $\mathbf{X}_i$, i.e. $\mathbf{X}_i^{(t-1)}$ and the past of process $\mathbf{X}_{A}^{(t-1)}$. Hence, DI testing boils down to testing whether both the mean squared variances are equal. Therefore, we form the mean squared test statistic in two ways leading to two different DI testers.

{\emph{\bf Test 1$(i,j, A)$}}: We follow the standard approach used in Granger causal studies \cite{barnett2014mvgc}. If $\Phi$ is constant, we create $T-1 \times 2$ matrix consisting of rows $(X_{i,t},X_{i,t+1})_{t=1}^{T-1}$  by taking all consecutive pairs from every time series stacking them vertically. Now, we stack these matrices vertically again to create an $NT-1 \times 2$ matrix $\mathbf{\tilde{X}}_i$. Let $\mathbf{\tilde{X}}_i[1,:]$ refer to the first column and let $\tilde{X}_i[2,:]$ refer to the second column. We solve the following two approximations through ordinary least squares regression:
\begin{enumerate}
\item $\min \lVert \mathbf{\tilde{X}}_j[2, :1] -\sum_{l \in A} \tilde{\Phi}^1_{lj} \mathbf{\tilde{X}}_{l}[1,:] \rVert_2$
\item $\min \lVert \mathbf{\tilde{X}}_j[2, :1] - \sum_{l \in A \cup i} \tilde{\Phi}^1_{lj} \mathbf{\tilde{X}}_{l}[1,:] \rVert_2$
\end{enumerate}

Let $\text{mse}_1$ be the mean squared error for the first least-squares approximation and $\text{mse}_2$ be the mean squared error for the second approximation. Then $(NT-1)\ln(\frac{\text{mse}_1} {\text{mse}_2})$ follows a $\chi^2$ distribution with $1$ degree of freedom and the p-value corresponding to the null hypothesis corresponds to the p-value of the null hypothesis $I(X_i \to X_j \| X_A) = 0$ (when the process is stationary and jointly Gaussian). This is, therefore, a DI testing oracle and we call it \textbf{Tester 1}.

{\emph{\bf Test 2$(i,j, A)$}}: The issue with 
\textbf{Tester 1} is that it does regression with highly autocorrelated samples of the same time series stacked vertically. This is a good practice when the number of i.i.d.\ copies $N$ is small. However, when $N$ is comparable to $T$, autocorrelation amongst a specific process would decrease the performance of the tester. Instead of regression through stacking as in the previous case, we compute the asymptotic prediction error as follows. 

1) We do two separate regressions for each pair of time points $(t+1,t)$ with $N$ i.i.d.\ samples: one using variable $i,A$ as a covariate to predict $j$ and another without $i$ and only with set $A$. 

2) Now consider the $N$ residues obtained after the regressions as $\epsilon_{j,A\cup i,t,n},~1 \leq n \leq N$ for the first regression. Similarly, let the residues for the second regression be $\epsilon_{j,A,t,n}$.

3) Denote $\Sigma_{j,A \cup i}$ to be the covariance matrix whose entries are indexed by $(t_1,t_2),~ t_1 \in [1:T],~t_2 \in [t:T]$. $\Sigma_{j,A \cup i}[t_1,t_2]$ is the covariance between $\epsilon_{j,A \cup i,t_1,\cdot}$ and $\epsilon_{j,A \cup i, t_2,\cdot}$ averaged over the $N$ i.i.d samples. Similarly, let $\Sigma_{j,A}[t_1,t_2]$ is the covariance between matrix calculated from the $\epsilon_{j,A,\cdot}$ variables. Now, since all variables are jointly Gaussian, the residues are also jointly Gaussian. Let $\Sigma_{j,A}^{(t)}$ be the covariance sub-matrix involving points with time index until $t$. Therefore, we compute the asymptotic prediction error given by the expression \cite{amblard2009relating}:
$
 \epsilon_\infty^2(j,A) \approx\frac{\det{\Sigma_{j,A}^{(T)}}}{\det{\Sigma_{j,A}^{(T-1)}}}.$
 Similar expressions hold for $\epsilon_\infty^2(j,A \cup i)$. This is motivated by the fact that for jointly Gaussian variables $x_1 \ldots x_n$, the squared prediction error of $x_n$ given the other is $\frac{\det{\Sigma^{(n)}}}{\det{\Sigma^{(n-1)}}}$, where $\Sigma^{(t)}$ is the bottom right $t \times t$ sub-matrix of $\Sigma$.
4) Under the null hypothesis, $(N-1)* [\log (\epsilon^2_{\infty}(j,A)) - \log (\epsilon^2_{\infty}(j,A \cup i))]$ is distributed with $\chi^2$ with $1$ degree of freedom. We call this \textbf{Tester 2}.

\section{Comparative Study}
We perform a comparative study similar to reference \cite{cooklearning}. We compare the results of MMPC-p to modified GC \cite{granger1969investigating,cooklearning,barnett2014mvgc} and modified PC \cite{shafer1995p,cooklearning}. For MMPC-p and PC we use \textbf{Tester 1} and \textbf{Tester 2}; for modified GC we use only \textbf{Tester 1}. We fix an $\alpha$ value for all the testers. For the sake of brevity, we refer to modified GC and modified PC as GC and PC, respectively in the remainder of this paper.

\textbf{Synthetic Datasets}: We generate synthetic datasets as described in \cite{cooklearning}. For a given density $\rho$ and number of nodes $N$, we generate 50 datasets consisting of directed graphs of $N$ nodes that contain at least one $N$-cycle, with coefficients of the AR(1) model in $\pm [0.2, 0.8]$ (before normalizing by the largest eigenvalue), such that the matrix has a density $\rho$. This method of constructing AR(1) models will generate matrices with very small eigenvalues that is fixed by adding a scaled identity to the AR(1) model (essentially adding feedback loops $X_i[t-1] \to X_i[t]$). We do this 50 times for each of $N = 10, 15, 20, 25, 30, 50$ and for densities $\rho=0.1,0.2,0.3$. For each of the datasets we generate 1000 samples, in the form of $N_{\text{series}}$ time series with $N_{\text{samples}}$ samples each, such that $N_{\text{series}}N_{\text{samples}} = 1000$. 

\textbf{Metrics}: We consider \textit{omission error rate}: false negative edges normalized by the total number of edges and \textit{commission error rate}: false positive edges normalized by the total number of non-edges.

\textbf{Discussion}: 
The leftmost plots in Figure \ref{fig:com-test1}, Figure \ref{fig:om-test1}, Figure \ref{fig:com-test2}, and Figure \ref{fig:om-test2} indicate that for large and sparse graphs both the omission and commission errors are well controlled for MMPC-p. The rightmost plots suggest that when the density is higher, commission errors for MMPC-p are still well controlled but the omission error increases. 

We conduct analysis for 50 variables separately in Figure~\ref{fig:50vars}. The results indicate that PC has high commission error with \textbf{Tester 1}, and GC cannot run because of the conditioning set being large. However, both commission and omission errors for MMPC-p with \textbf{Tester 1} are lower than PC and GC. 
\begin{figure}
\centering
\includegraphics[width=\linewidth]{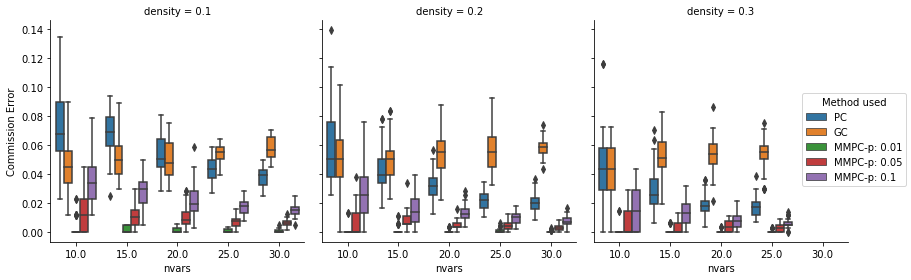}
\caption{Commission errors for all methods with \textbf{Tester 1}.}
\label{fig:com-test1}
\end{figure}
\begin{figure}
\centering
\includegraphics[width=\linewidth]{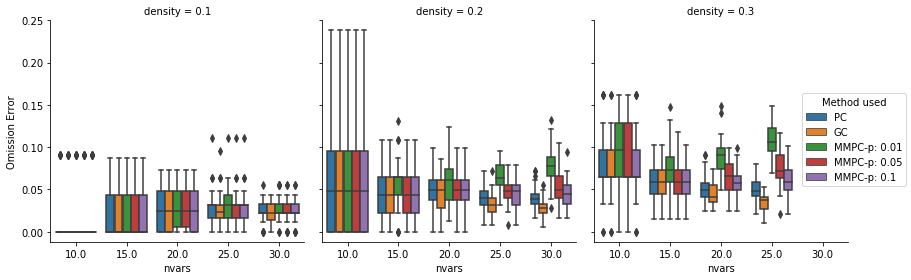}
\caption{Omission errors for all methods with \textbf{Tester 1.}}
\label{fig:om-test1}
\end{figure}

\begin{figure}
\centering
\includegraphics[width=\linewidth]{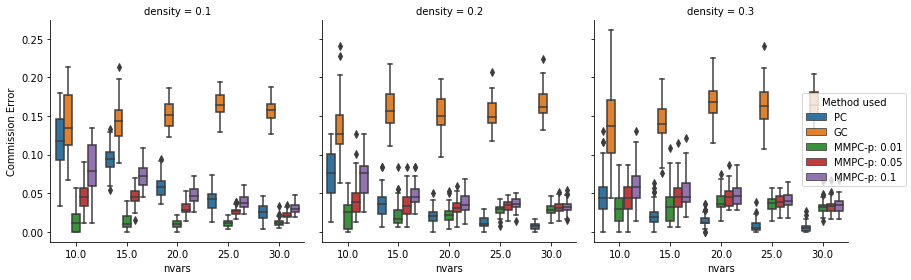}
\caption{Commission errors for all methods with \textbf{Tester 2}, except for GC where we use \textbf{Tester 1}.}
\label{fig:com-test2}
\end{figure}
\begin{figure}
\centering
\includegraphics[width=\linewidth]{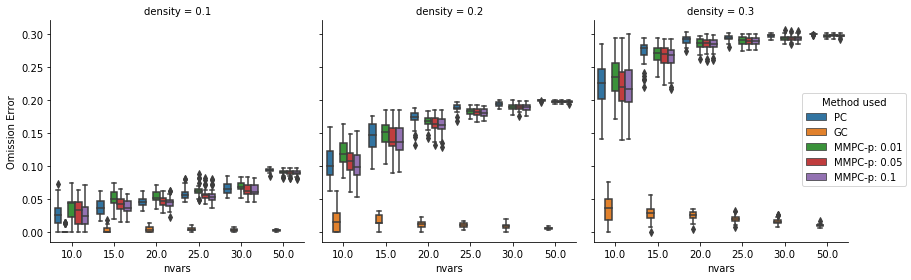}
\caption{Omission errors for all methods with \textbf{Tester 2}, except for GC where we use \textbf{Tester 1}.}
\label{fig:om-test2}
\end{figure}

\begin{figure}
\begin{subfigure}{.5\textwidth}
 \centering
 \includegraphics[width=\linewidth]{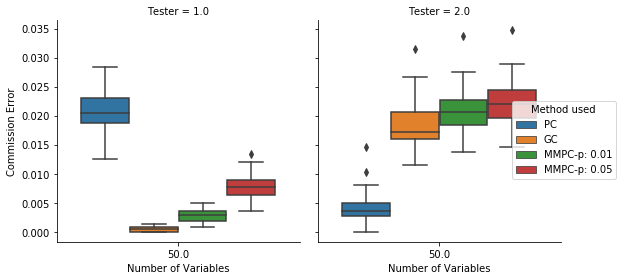}
\end{subfigure}%
\begin{subfigure}{.5\textwidth}
 \centering
 \includegraphics[width=\linewidth]{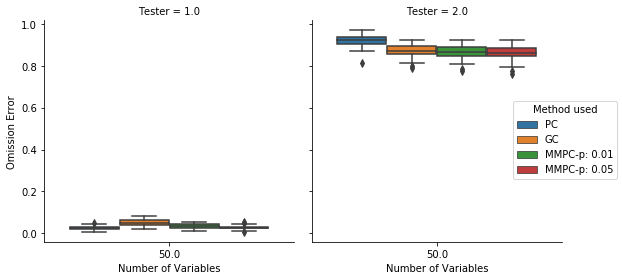}
\end{subfigure}
\caption{\small Commission and omission errors for all methods \textbf{Tester 1} and \textbf{Tester 2} with 50 variables.}
\label{fig:50vars}
\end{figure}

\section{Global Development Case Study}
We consider a dataset of over 4,000 random processes, most of which begin in 1960, across a wide range of development issue areas. The source is largely international organizations like the World Bank, the Food and Agriculture Organization of the United Nations, the UNESCO Institute for Statistics, and the International Monetary Fund and has been standardized (all series structured identically with detailed metadata) in the IFs platform, a free, open-source long-term global integrated assessment model. For each process in the dataset, we have 186 time series samples, each of them corresponding to a country. The length varies between 50 and 60 time steps. 

We compare the results of MMPC-p with the causal connections used in the IFs model. The connections in IFs are the result of a deductive approach, i.e.\ domain knowledge based on academic literature and conceptually sound statistics. 
In Table~\ref{tab:mmpctester1results}, we show the parents obtained for four series (selected arbitrarily): AGCropProductionFAO, GDPCurDol, LaborAgricultureTotMale, and Population.
\begin{table}[ht]
    \centering
    \caption{Parents obtained through MMPC-p using \textbf{Tester 1}.}
    \label{tab:mmpctester1results}
    \scalebox{1}{
        \begin{tabular}{|l|p{0.3\linewidth}|p{0.3\linewidth}|}
            \hline
            Series & Parent (max p-value)  & Parent (IFs) \\  \hline
            AGCropProductionFAO & AGCroptoFoodFAO (1e-9), Market for PC sales (2.22e-6) & Land Crop, Change in Precipitation, Annual Temperature Change, Land Equipped for Irrigation, Labor in Agriculture, Arable Land\\ \hline
            GDP Current Dollars & GDP (1e-9) & Labor\\ \hline
            LaborAgricultureTotal\%Male & Revenue Contribution (3.1e-5) & Value added Agriculture\\ \hline
            Population & Internet Subscribers (1e-9), Cooking Oil, Fuel and Coal (1e-9) & Birth, Death, Migration\\ \hline
        \end{tabular}}
    \end{table}

For the most part, the parents identified by MMPC-p do not match the causal drivers in IFs. They are a mixed bag: some are semantically similar to the IFs parents, such as AGCroptoFoodFAO and GDP; we have validated them to be semantically similar with domain experts.  Others like Market for PC sales, Internet Subscribers, and Cooking Oil are spurious.  

One of the main reasons for the mismatch between MMPC-p and the IFs model is that many variables used in the model do not have a direct corresponding data series. For example, one of the two direct drivers of crop production is yield, measured as tons per hectare. But the variable for yield used in the IFs model is initialized using data series for crop production and crop land (the quotient being yield). So, MMPC-p is unable to identify yield as a direct parent of crop production. Another reason for the mismatch is that the dataset used by MMPC-p contains many series that are aggregated for use in the IFs model, and are not directly causally connected to other variables. For example, calories per capita is an important development indicator, and a direct driver of hunger, but is initialized in the IFs model through the sum of ten series for calories per capita from different food sources. Finally, a technical reason for the mismatch could be that many of these relations are non-linear in nature. Other testers that do not require linearity could be used with more available data and perhaps yield results more similar to IFs.

\section{Conclusion}
\label{sec:conclusion}

In this paper, we have proposed a new algorithm for learning the Granger causal structure of observational time series and endowed it with strong FDR control.  Named MMPC-p, it is inspired by the hill-climbing MMHC approach for causal structure learning in non-time series observations and inherits its scalability to large numbers of random processes.  We conduct a comprehensive comparison to GC and PC with two different DI testers on large sparse graphs, finding that the proposed algorithm has better FDR control and scalability than the competing algorithms.  We have also taken the first steps to using the algorithm in practice for a global development use case as an alternative to years-long modeling efforts. Our results are observed to be semantically similar for some variables when compared to the existing ground-truth. There is still room for improvement in better aligning with international studies practice; in fact, one piece of future work is to use the human-validated relationships not only as a comparison point for validating algorithm outputs, but as input for an improved algorithm that is a hybrid of deduction and data-driven inference.

\bibliographystyle{IEEEtran}
\bibliography{grangercausal}

\end{document}